\documentclass{article}
\usepackage{etex}
\usepackage{spconf,amsmath,graphicx}

\usepackage{endnotes}

\usepackage{epsfig,psfrag}
\usepackage{pst-all}
\usepackage{amsmath,amsthm,amssymb,amsfonts,upref,cite,epsf,color,bm}
\usepackage{graphicx}
\usepackage{color}
\usepackage{amsmath}
\usepackage{graphicx}
\usepackage{calc}

\usepackage{tikz}
\usepackage{pgfplots}

\newtheorem{theorem}{Theorem}[section]

\newtheorem{lemma}[theorem]{Lemma}

\newtheorem{proposition}[theorem]{Proposition}

\newtheorem{algorithm}{Algorithm}

\usepackage{framed}

\newcommand\norm[2][\Tnorm]{\ensuremath{{\|#2\|}_{#1}}}

\newcommand{\pinv}[1]{  {#1}^{ \dagger } } % Pseudoinverse
\newcommand{\inv}[1]{  {#1}^{ -1 } } % Inverse of a matrix
\newcommand{\herm}[1]{{#1}^H} % conjungate transpose of a matrix
\newcommand{\transp}[1]{{#1}^T} % transpose of a matrix
\newcommand\defeq{:=}

\newcommand{\vbed}{\bm \beta} % I substituted the \vbe' by \bm beta, so there are not double superscripts,

\newcommand{\vep}{\bm \varepsilon}
\newcommand{\vbe}{\bm \beta}

\newcommand\vect[1]{\mathbf #1}

\newcommand{\vc}{\vect{c}}

\newcommand{\vx}{\vect{x}}  
\newcommand{\vy}{\vect{y}}

\newcommand{\mG}{\vect{G}}

\newcommand{\mI}{\vect{I}}

\newcommand{\mX}{\vect{X}}

\renewcommand{\S}{\mathcal S}

 \newcommand{\complexset}{\mathbb C}
 
 \newcommand{\mc}{\mathcal}

\newcommand\comp[1]{ {#1}^c}

\usepackage[applemac]{inputenc}

\small\normalsize
\def \prob {{\rm P} }
\def \twiddle[#1] {e^{-j \frac{2 \pi}{N}  #1 }}
\def \twiddleneg[#1] {e^{j \frac{2 \pi}{N}  #1 }}
\DeclareMathOperator{\supp}{supp}

\DeclareMathOperator*{\argmin}{arg\;min}

\DeclareMathOperator*{\diag}{diag}
\DeclareMathOperator*{\EX}{E}

\DeclareMathOperator*{\spn}{span}

\def\ML_est{\hat{\mathbf{x}}_{\text{ML}}}

\newcommand{\autocovfunc}{\mathbf{R}} 
 
\newcommand{\cig}{\mathcal{G}}

\newcommand{\specdensmatrix}{\mathbf{S}}  % ohne x finde ich besser; wir haben eh so viele subscripts 
\newcommand{\ESDM}{\widehat{\mathbf{S}}}  
\newcommand{\EACF}{\widehat{\mathbf{R}}}
\newcommand{\processclass}{\mathcal{M}(s_{\text{max}},\rho_{\text{min}},\mu^{(h_{1})},\phi_{\text{min}},A,B)}

\newcommand{\be}{\begin{equation}}
\newcommand{\ee}{\end{equation}}
\newcommand{\ist}{\hspace*{.2mm}}
\newcommand{\rmv}{\hspace*{-.2mm}}
\newcommand{\lagvar}{m}

\newcommand{\phimin}{\phi_{\text{min}}}

\newcommand{\nrtasks}{F} 
\newcommand{\task}{f}

\allowdisplaybreaks

\makeatother

\usepackage{setspace}
\newlength{\depthofsumsign}
\title{Compressive Nonparametric Graphical Model Selection \\ for Time Series}
\name{Alexander Jung$^{\ist 1}\rmv\rmv$, Reinhard Heckel$^{\ist\ist 2}\rmv\rmv$, Helmut B{\"o}lcskei$^{\ist 2}\rmv\rmv$, 
and Franz Hlawatsch$^{\ist\ist 1}$\thanks{\hspace*{-5mm}The work of F. Hlawatsch was supported by the Austrian Science Fund (FWF) under Grant S10603.} 
}
\address{\normalsize $^1$Institute of Telecommunications, Vienna University of Technology, Austria; \{ajung,\,fhlawats\}@nt.tuwien.ac.at\\[-0.5mm]
\normalsize $^2$Dept. IT \& EE, ETH Zurich, Switzerland; \{heckelr,\,boelcskei\}@nari.ee.ethz.ch\\[-0.5mm]
}
\begin{document}
\maketitle


%\maketitle

%% \renewcommand{\baselinestretch}{0.87}\small\normalsize

\begin{abstract}
%\boldmath
We propose a method for inferring the conditional independence graph (CIG) of a high-dimensional discrete-time Gaussian vector random process from 
finite-length observations. 
Our approach does not rely on a parametric model (such as, e.g., an autoregressive model) for the vector random process; rather, it only assumes  
certain spectral smoothness properties. The proposed inference scheme is compressive in that it works for sample sizes that are (much) smaller than the number of scalar process components. % of the vector process. 
We provide analytical conditions for our method to correctly identify the CIG with high probability. %, and to be asymptotically consistent.  
%Our analysis also reveals conditions for the new method to be asymptotically consistent. 
\end{abstract}

\vspace*{-1mm}
%%%%%%%%%%%%%%%%%%%%%%%%%%%%%%%%%%%
\section{Introduction}
\label{sec_intro} 
%%%%%%%%%%%%%%%%%%%%%%%%%%%%%%%%%%%

\vspace{-2mm}

Consider a $p$-dimensional, zero-mean, stationary, Gaussian random process $\mathbf{x}[n] \in \mathbb R^p$, $n \in \mathbb{Z}$. 
We are interested in learning the conditional independence graph (CIG)  \cite{Dahlhaus2000,DahlhausEichler2003,BachJordan04,PHDEichler} of $\mathbf{x}[n]$ from  
the finite-length observation of a single process realization. We consider the \emph{high-dimensional regime}, where the number of observed process samples is (much) smaller than $p$  \cite{ElKaroui08,Santhanam2012,RavWainLaff2010,Nowak2011,Bento2010,MeinBuhl2006,FriedHastieTibsh2008}. 
In this case, consistent estimation of the CIG is possible only if structural assumptions on the vector process are made.
Specifically, we will consider CIGs that are \emph{sparse} in the sense of containing relatively few edges. This problem is relevant, e.g., in the analysis of the time evolution of air pollutant concentrations \cite{Dahlhaus2000,DahlhausEichler2003} 
and in medical diagnostic data analysis \cite{Nowak2011}. 

%{\it State of the art:}
Existing approaches to this compressive graphical model selection problem are based on parametric process models \cite{Nowak2011,Bento2010,Songsiri09,songsiri2010}, specifically on vector 
autoregressive (VAR) models. In this paper, we develop and analyze a nonparametric  approach, which only requires the vector process to be 
spectrally smooth.  %in the spectral domain. 
The smoothness notion we use is quanti\-fied 
in terms of moments of the matrix-valued auto-covariance function (ACF) of the process. 
Compared to \cite{MeinBuhl2006,Nowak2011,Bento2010,Songsiri09,songsiri2010}, our approach applies to a considerably more general 
class of processes including VAR processes as a special case. 

\vspace{-3mm}

\paragraph*{Contributions:}
Our main conceptual contribution resides in recognizing that the problem of inferring the sparse CIG of a Gaussian vector process is a special case of a \emph{block-sparse signal recovery problem} \cite{BlockSparsityEldarTSP,MishaliEldar2008,EldarRauhut2010}, i.e., a \emph{multitask learning problem} \cite{BuhlGeerBook,Lounici09}. While for the special case of a VAR process with sparse CIG, a block-sparse structure was already identified in \cite{Nowak2011}, we show that a (different) block-sparse structure exists for general stationary time series. This stems from the fact that the CIG of a general stationary time series 
is encoded in the continuous ensemble of values of the spectral density matrix $\specdensmatrix(\theta)$, $\theta \in [0,1)$. 
Based on this insight, we develop a \emph{multitask LASSO} \cite{BuhlGeerBook,Lee_adaptivemulti-task} formulation of the sparse CIG estimation problem. 
Our main analytical contribution is Theorem \ref{thm_main_result_consistency_CS_graph_model}, which provides conditions for our scheme to correctly identify the CIG with high probability. %, and and to be asymptotically consistent. 

\vspace{-0.4cm}
\paragraph*{Outline:}
The remainder of this paper is organized as follows. 
%In Section~\ref{sec:probstate}, we formalize the problem considered, and
In Section~\ref{SecProblemFormulation}, we introduce the problem considered. Section \ref{sec_graph_model_sel_multitask_learning} describes our CIG inference method and Section \ref{sec_var_sel_consist} presents  corresponding performance guarantees. Finally, Section \ref{sec_numerical_results} reports numerical results. 
\vspace{-0.15cm}
%\paragraph*{ Notation:} 
%We use lowercase boldface letters to denote (column) vectors and uppercase boldface letters to designate matrices. For the vector $\vx$, $x_q$ stands for its $q$th entry. For the matrix $\mA$, $\mA_{ij}$ is the entry in its $i$th row and $j$th column, and $\pinv{\mA}$ its pseudo-inverse.  
%The set $\{1,...,N\}$ is denoted by $[N]$ and the cardinality of the set $\S$ is $|\S|$. 
%We write $\mathcal N( \boldsymbol{\mu},\boldsymbol{\Sigma})$ for a Gaussian random vector with mean $\boldsymbol{\mu}$ and covariance matrix $\boldsymbol{\Sigma}$. 
%The unit sphere in $\reals^m$ is $\US{m} \defeq \{ \vx \in \reals^m \colon \norm[2]{\vx} = 1 \}$. %Finally, we use the shorthand $d_{\max} = \max_l d_\l$. 

\section{Problem Formulation}
\label{SecProblemFormulation}
\vspace*{-2mm}
%%%%%%%%%%%%%%%%%%%%%%%%%%%%%%%%%%%%%%%%%%%%%%%%%%%%%%%%%%

Consider a $p$-dimensional, zero-mean, stationary, real, Gaussian random process $\mathbf{x}[n]$ with (matrix-valued) ACF $\autocovfunc[\lagvar] \defeq \EX\{ \vx [\lagvar] \transp{\vx}[0] \}$. The ACF is assumed summable, i.e., $\sum_{\lagvar = -\infty}^{\infty} \! \norm{\autocovfunc[\lagvar]}\!<\!\infty$ for some matrix norm $\|\!\cdot\!\|$. %\footnote{Note that, for fixed dimensions, all matrix norms are equivalent \cite{golub96}. Therefore, if the ACF is summable under any matrix norm, it is summable under every matrix norm.} 
The spectral density matrix (SDM) of the process $\mathbf{x}[n]$ is defined as 
%\begin{equation}
%\label{equ_def_spectral_density_matrix}
$\specdensmatrix(\theta) \! \defeq \! \sum_{\lagvar=-\infty}^{\infty} \autocovfunc[\lagvar] \exp(-j2\pi \theta \lagvar) \in \mathbb C^{p\times p}$, %= \sum_{k = - \infty}^{\infty} \EX{ \mathbf{x}[k] \mathbf{x}^{T}[0] } \exp(-j2\pi \theta k)
and we assume 
\vspace*{-1.2mm}
that
\begin{equation}
\label{equ_uniform_bound_eigvals_specdensmatrix}
%0 \! < \! A \! \leq \! \lambda_{\text{min}}(\specdensmatrix(\theta)) \! \leq \! \lambda_{\text{max}}(\specdensmatrix(\theta)) \! \leq \! B \!<\! \infty \mbox{, }  \forall \theta \! \in \! [0,1).
0 < A  \leq  \nu_{\text{min}}(\specdensmatrix(\theta)) \leq \nu_{\text{max}}(\specdensmatrix(\theta)) \leq B < \infty 
\vspace*{-1mm}
\end{equation} 
for all $\theta \in [0,1)$, where $\nu_{\text{min}}(\specdensmatrix(\theta))$ and $\nu_{\text{max}}(\specdensmatrix(\theta))$ denote the smallest and largest eigenvalue of $\specdensmatrix(\theta)$, respectively. In particular, \eqref{equ_uniform_bound_eigvals_specdensmatrix} implies that the matrix $\specdensmatrix(\theta)$ is nonsingular for all $\theta$. 
%Moreover, we require the SDM entries $\big[ \specdensmatrix(\theta) \big]_{k,l}$ to be smooth functions of $\theta$. 
We will furthermore require the vector process $\mathbf{x}[n]$ to be such that $\specdensmatrix(\theta)$ satisfies certain smoothness properties which are expressed in terms of moments of the ACF defined as
%This property is essential to allow for accurate estimators of the SDM. 
%Based on the Fourier relationship between the SDM and the ACF,
%we characterize the degree of smoothness of $\specdensmatrix(\theta)$ by the ACF moment 
\vspace*{-1mm}
\begin{equation}
\label{equ_def_generic_moments_ACF}
\mu^{(h)} \defeq \sum_{\lagvar=-\infty}^{\infty} h[\lagvar]  \| \autocovfunc[\lagvar] \|_{\infty}. 
\vspace*{-.8mm}
\end{equation}
Here, $h[m]$ is a nonnegative weight function that typically increases with $|\lagvar|$. 

The CIG of the process $\vx[n]$ is the graph $\cig \!\defeq\! (V,E)$ with node set $V=[p] \! \defeq \! \{1,\ldots,p\}$ representing the scalar component processes $\{ x_{r}[n] \}_{r \in [p]}$ and edge set $E \subseteq [p] \times [p]$, where $(k,l) \notin E$ if and only if the component processes $x_{k}[n]$ and $x_{l}[n]$ are conditionally independent 
\pagebreak %%%%%%
given 
all remaining component processes $\{ x_{r} [n] \}_{r \in [p] \setminus \{ k,l \} }$ \cite{Dahlhaus2000}. 
%By convention, a CIG contains no self-loops  \cite{LauritzenGM}, i.e., for every $r \in V$ we have $(r,r) \notin E$.
The neighborhood of node $r \in [p]$ is defined as 
$\mathcal{N}(r) \! \defeq \! \{ r' \in   [p] \,\, | \,\,  (r,r') \in E \}$.
We restrict ourselves to processes with sparse CIG $\cig$ in the sense that
\vspace*{-2mm}
\begin{equation}
\label{equ_def_maximum_node_degree}
\max_{r \in [p]} | \mathcal{N}(r) | \leq s_{\text{max}}   \ll p. 
\vspace*{-2mm}
\end{equation} 
%i.e., every process component depends on few other components. 
%i.e., every node depends only on few other nodes. 
%satisfies $s_{\text{max}} \ll p$. 
%
The graphical model selection problem we consider can now be stated as the problem of inferring the CIG $\cig$, or more precisely its edge set $E$, from the observation $\big( \mathbf{x}[1],\ldots,\mathbf{x}[N] \big)$, where $N$ is the sample size.
%Graphical model selection refers to the task of inferring the CIG $\cig$ of the process $\mathbf{x}[n]$, based on an observed finite length data block. 
Since %\cite{Dahlhaus2000,DahlhausEichler2003,Brillinger96remarksconcerning}, since we assume 
$\mathbf{x}[n]$ is Gaussian with $\specdensmatrix(\theta)$ nonsingular for all $\theta \in [0,1)$, it follows from \cite{Dahlhaus2000,DahlhausEichler2003,Brillinger96remarksconcerning} that % the component processes $x_{k}[n]$ and $x_{l}[n]$ are conditionally independent given all remaining component processes $\{ x_{r}[n] \}_{r \in [p] \setminus \{k,l\}}$ 
$(k,l) \notin E$ if and only if $\big[ \specdensmatrix^{-1}(\theta) \big]_{k,l}=0$ for all $\theta \in [0,1)$. The edge set $E$ therefore corresponds to the locations of the nonzero entries of $\specdensmatrix^{-1}(\theta)$, and our graphical model selection problem amounts to determining these locations. % of these nonzero entries. % of $\specdensmatrix^{-1}(\theta)$. 
We are interested in estimating the CIG $\cig$ from $\nrtasks$ regularly spaced samples $\{ \specdensmatrix(\theta_{\task}) \}_{\task \in [\nrtasks]}$, with 
$\theta_{\task} \defeq (\task-1)/\nrtasks$, $\task \in [\nrtasks]$, and 
%In particular, our selection scheme will be based solely on an estimate of the SDM values $\specdensmatrix(\theta_{\task})$ with $\theta_{\task} = (\task-1)/\nrtasks$ for $\task \in [\nrtasks]$. 
 $\nrtasks$ large enough for the following to hold:
\begin{equation}
\label{equ_charact_edge_set_covariance_matrix_k}
(k,l) \notin E  \Longleftrightarrow  \big[\specdensmatrix^{-1}(\theta_{\task})\big]_{k,l}=0 \mbox{\, for all }  \task \in [\nrtasks].
\vspace{-.5mm}
\end{equation}  
The implication from left to right in \eqref{equ_charact_edge_set_covariance_matrix_k} follows trivially from what was said above. The implication from right to left is satisfied, e.g., for processes with all entries of $\specdensmatrix(\theta)$ being rational functions in $\exp(j\theta)$, provided that $\nrtasks$ is larger than the maximum degree of the numerator polynomials of $\specdensmatrix^{-1}(\theta)$. 
Another sufficient condition for the implication from right to left to hold is the following. % can be based on the boundedness condition 
\vspace*{-1.5mm}
\begin{lemma}
Consider a $p$-dimensional, zero-mean, stationary, Gaussian process $\mathbf{x}[n]$ with CIG $\cig$ and SDM $\specdensmatrix(\theta)$ satisfying \eqref{equ_uniform_bound_eigvals_specdensmatrix}. 
Then, if $\nrtasks$ is chosen such that for every edge $(k,l) \in E$, the ACF moment $\mu^{(h_{0})}$ with $h_{0}[m]\!=\!|m|$ and the global partial coherence $\Gamma^{(k,l)} \defeq \int_{0}^{1} \Big| \big[\specdensmatrix^{-1}(\theta)\big]_{k,l} \Big/\sqrt{\big[ \specdensmatrix^{-1}(\theta) \big]_{k,k}\big[ \specdensmatrix^{-1}(\theta) \big]_{l,l}} \Big| d \theta$ satisfy $\mu^{(h_{0})} / (A \Gamma^{(k,l)})<\nrtasks$, with $A$ as in \eqref{equ_uniform_bound_eigvals_specdensmatrix}, 
the CIG $\cig$ is characterized by \eqref{equ_charact_edge_set_covariance_matrix_k}.
\vspace*{-1.5mm}
\end{lemma}

The restriction to the finite set of frequencies $\{\theta_{\task}\}_{\task \in [\nrtasks]}$ is made for expositional convenience. The general theory developed in this paper goes through for $\theta \in [0,1)$, with our inference procedure becoming a  multitask learning problem with a continuum instead of a finite number, $F$, of tasks. 
%with a continuum of tasks instead of $F$ tasks. 
%The restriction to the finite set of frequencies $\{\theta_{\task}\}_{\task \in [\nrtasks]}$ %in Theorem \ref{thm_neighborhood_node_r_permuted_generalized_support_multitask_learning} and \eqref{equ_charact_edge_set_covariance_matrix_k} 
%is somewhat arbitrary. We use this discretization mainly because it allows for a straightforward formulation and analysis of our CIG selection scheme. 
%As mentioned before, we will formulate our sparse CIG estimation problem as a multitask learning problem with $F$ tasks. In principle,  we could define this multitask learning problem for all $\theta \in [0,1)$, thereby obtaining a multitask learning problem with a continuum of tasks. 
%The selection problem then becomes a block-sparse recovery problem with blocks belonging to infinite-dimensional spaces. 
%This intrinsic infinite-dimensional structure for general stationary processes is a decisive difference from the finite-dimensional block-sparse structure obtained in \cite{Nowak2011} for VAR processes.

\vspace{-0.2cm}
\section{Graphical Model Selection} 
\label{sec_graph_model_sel_multitask_learning}
\vspace*{-2mm}

Our method for inferring the CIG $\cig$ is inspired by the approach employed in \cite{MeinBuhl2006,RavWainLaff2010}. We first estimate the SDM $\specdensmatrix(\theta_{\task})$, $\task \in [\nrtasks]$, by means of a multivariate spectral estimator. Then we use this estimate to perform \emph{neighborhood regression},  % \cite{MeinBuhl2006,RavWainLaff2010}. 
which yields an estimate of the support (i.e., the locations of the nonzero entries)  of $\specdensmatrix^{-1}(\theta)$ and, via \eqref{equ_charact_edge_set_covariance_matrix_k}, the CIG. 
%Specifically, based on the estimate $\specdensmatrix(\theta_{\task})$, 
Neighborhood regression is performed by solving a multitask learning problem using multitask LASSO (mLASSO) \cite{BuhlGeerBook}. 
%The resulting overall algorithm essentially solves a multitask learning problem using multitask LASSO (mLASSO) \cite{BuhlGeerBook}. 

%of a stationary vector process can be inferred by combining an established estimator for the SDM with a \emph{compressed sensing} recovery method, i.e., mLASSO. 
%To this end, we first assume perfect knowledge of the SDM values $\specdensmatrix(\theta_{\task})$ and show in Section \ref{mult_learning_form_gm_sel} how the general concept of \emph{neighborhood regression} \cite{MeinBuhl2006,RavWainLaff2010} leads to a multitask learning problem in our setting. This yields, in turn, a mLASSO based method 
%for inferring the neighborhood $\mathcal{N}(r)$. By addressing the fact that the exact SDM values are unknown and have to be estimated, we then obtain our novel selection scheme in Section \ref{subsec_novel_sel_scheme}.  

\newcommand{\mCx}{\mathbf{C}}
\newcommand{\f}{f}

%\subsection{Multitask Learning Formulation}
%\label{mult_learning_form_gm_sel}
%\vspace*{-2mm}

%As mentioned before, we will estimate the support of $\specdensmatrix^{-1}(\theta_{\task})$, as this yields the CIG via \eqref{equ_charact_edge_set_covariance_matrix_k}. To this end, we propose the following two step selection scheme:\footnote{The approach in \cite{MeinBuhl2006} uses this two step procedure implicitly, but the authors do not discuss this fact.}  First, an estimate $\ESDM(\theta)$ of $\specdensmatrix(\theta)$ is obtained from the observations $(\mathbf{x}[1],\ldots,\mathbf{x}[N])$ of the process $\mathbf{x}[n]$. Based on the estimate $\ESDM(\theta)$, an estimate of $\specdensmatrix^{-1}(\theta)$ is obtained. 
%$\widehat{\mathcal{N}}(r)$ of the neighborhood $\mathcal{N}(r)$ of each node $r \in [p]$ is computed.%\footnote{Finally, following the approach in \cite{MeinBuhl2006},  the individual neighborhood estimates may be combined by taking either the union 
%$\bigcup_{r \in [p]}  \widehat{\mathcal{N}}(r)$ or the intersection $\bigcap_{r \in [p]}  \widehat{\mathcal{N}}(r)$ to obtain an estimate $\hat{E}$ of the unknown edge set $E$ of the CIG $\cig$.}

With regards to the first step, it is natural to estimate $\specdensmatrix(\theta)$ using the multivariate Blackman-Tukey 
\vspace{-1mm}
estimator \cite{stoi97}: 
\begin{equation} 
\label{equ_est_BT_sdm_theta_t}
\ESDM(\theta)  \defeq \sum_{\lagvar=-N+1}^{N-1} w[\lagvar] \EACF[\lagvar]  \exp(-j 2 \pi \theta \lagvar). 
\end{equation} 
Here, 
$\EACF[\lagvar]  \defeq  (1/N) \sum_{n =1}^{N-\lagvar} \vx[n+m] \transp{\vx}[n]$ for $\lagvar \in \{0,\ldots,N-1\}$ and, by symmetry of the ACF, $\EACF[\lagvar]  \defeq \EACF^{H}[-\lagvar]$ for $\lagvar \in \{-N+1,\ldots,-1\}$. 
Furthermore, the window function $w[\lagvar]$ is chosen %such that $w[\lagvar]=0$ for $|\lagvar|>N$ and 
such that $\ESDM(\theta)$ is positive semidefinite. All window functions with nonnegative discrete-time Fourier transform are admissible \cite[Sec.~2.5.2]{stoi97}. %  independently of the process .  
%A sufficient condition for this to happen is non-negativity of the discrete time Fourier transform of the window function $w[\lagvar]$ \cite{stoi97}. %i.e., 
%The estimator \eqref{equ_est_BT_sdm_theta_t} is a trivial generalization of the Blackman-Tukey estimator \cite{stoi97} for the power spectral density of a scalar stationary random process to the multivariate case.
%
%
In the high-dimensional regime, where the number $N$ of observations is smaller than the number $p$ of nodes, the matrices $\ESDM(\theta_{\task})$ in \eqref{equ_est_BT_sdm_theta_t} will be rank-deficient (to see this, note that each column of $\ESDM(\theta_{\task})$ is a linear combination of $\vx[n], n\in [N]$). 
Simply inverting $\ESDM(\theta_{\task})$, for $\task \in [\nrtasks]$, and inferring the edge set $E$ via \eqref{equ_charact_edge_set_covariance_matrix_k} is therefore not possible. 

\vspace{-3mm}

To cope with this issue, we reduce the problem of finding the support of the matrices $\specdensmatrix^{-1}(\theta_{\task})$ to \emph{multitask learning problems} (one for each node). This can be done as follows. 
First note that, because of \eqref{equ_charact_edge_set_covariance_matrix_k},  the union of the supports of the $r$th rows of the matrices $\specdensmatrix^{-1}(\theta_{\task}), f \in [F]$,  determines the neighborhood $\mathcal{N}(r)$. The
 $\mathcal{N}(r)$, as shown next, can then be obtained by solving multitask learning problems. 
 For simplicity of exposition and without loss of generality we assume $r=1$ in the following.  % and we use for notational convenience the shorthand $\specdensmatrix(\theta_\f) \defeq \specdensmatrix(\theta_\f)$.  
Given $\specdensmatrix(\theta_\f)$, $f\in [F]$, we define $\vy^{(\f)}\in \complexset^{p}$ and $\mX^{(\f)} \in \complexset^{p \times (p-1)}$ 
\vspace{-1mm}
via
\begin{equation}
[\vy^{(\f)} \, \; \mX^{(\f)}] \defeq \specdensmatrix^{1/2}(\theta_\f)
\label{eq:defyX}
\vspace{-1mm}
\end{equation}
where $\specdensmatrix^{1/2}(\theta_\f)$ is the positive definite square root of $\specdensmatrix(\theta_\f)$. 
We next decompose $\vy^{(\f)}$ into its orthogonal projection onto $\spn(\mX^{(\f)})$ and the orthogonal complement thereof according 
\vspace{.3mm}
to 
%and its orthogonal projection onto the complement of $\spn(\mX^{(\f)})$ according to 
%$\mX^{(\f)} \vbe^{(\f)}$ is the orthogonal projection of $\vy^{(\f)}$ onto $\spn(\mX^{(\f)})$, and $\vep^{(\f)}$ is the projection of $\vy^{(\f)}$ on the orthogonal complement of $\spn(\mX^{(\f)})$. 
%according to
%\begin{equation}
%\label{equ_def_signal_model_given_cov_matrix}
\begin{equation}
\vy^{(\f)} = \mX^{(\f)} \vbe^{(\f)} + \vep^{(\f)}, \quad f \in [F]
\label{eq:gmmvform}
\vspace{.4mm}
\end{equation}
%\end{equation}
with 
%\begin{equation}
%\label{equ_def_parameter_vector_given_cov_matrix}
$
\vbe^{(\f)} \defeq \pinv{\mX^{(\f)}} \vy^{(\f)}
$
\vspace*{-.7mm}
and 
$
\vep^{(\f)} \defeq (\mI - \mX^{(\f)} \pinv{\mX^{(\f)}})\vy^{(\f)}
$, 
where $\pinv{\mX^{(\f)}}$ is the pseudo-inverse of $\mX^{(\f)}$. 
%\end{equation}
%Thus, we have 
%
%$\vy^{(\f)}$ is decomposed into the component $\mX^{(\f)} \vbe^{(\f)}$, the orthogonal projection of $\vy^{(\f)}$ on  $\spn(\mX^{(\f)})$, and the component $\vep^{(\f)}$, the projection of $\vy^{(\f)}$ on the orthogonal complement of $\spn(\mX^{(\f)})$. 
The significance of this construction is expressed by the following 
proposition. 

\vspace*{-9.5mm}

\begin{proposition}
\label{prop:ladf}
The neighborhood $\mathcal N(1)$ of node $r=1$ is determined by the joint support 
of the $\vbe^{(\f)}, f \in [F]$, according to
%% \vspace{1mm}
\begin{equation}
\mathcal{N}(1) 
%&= \bigcup_{f\in [F]} \supp([\inv{\specdensmatrix}(\theta_\f)]_{2:p,1}) + 1 \nonumber \\
=  \bigcup_{f\in [F]} \supp(\vbe^{(\f)}) + 1
\label{eq:eqn1}
\vspace{0.5mm}
\end{equation}
where the addition in \eqref{eq:eqn1} is elementwise.
\vspace*{-9.5mm}
\end{proposition}

\begin{proof}
We first note that \eqref{equ_charact_edge_set_covariance_matrix_k} 
\vspace{-0.7mm}
implies 
\begin{equation}
\mathcal{N}(1) 
= \bigcup_{f\in [F]} \supp([\inv{\specdensmatrix}(\theta_\f)]_{2:p,1}) + 1 \nonumber 
\vspace{-0.2cm}
\end{equation}
where $[\inv{\specdensmatrix}(\theta_\f)]_{2:p,1}$ is the vector containing the entries $[\specdensmatrix^{-1}(\theta_{f})]_{2,1}, [\specdensmatrix^{-1}(\theta_{f})]_{3,1}, ... , [\specdensmatrix^{-1}(\theta_{f})]_{p,1}$. %$2,3,...,p$ of the first column of $\inv{\specdensmatrix}(\theta_\f)$. 
Next, we show that 
$
\supp([\inv{\specdensmatrix}(\theta_\f)]_{2:p,1}) = \supp(\vbe^{(\f)})
$, 
which will finalize the proof. 
%
%This can be seen as follows. 
By the construction of $\vy^{(\f)}$ and $\mX^{(\f)}$ in \eqref{eq:defyX}, we have
\begin{equation}
\label{equ_constrution_multitask_model_cov_matrix}
\specdensmatrix(\theta_\f) = 
\begin{bmatrix}
%\herm{\vy^{(\f)}} \vy^{(\f)} 
\norm[2]{\vy^{(\f)}}^2
&\hspace{-0.15cm} \herm{\vy^{(\f)}} \mX^{(\f)} \\[0.2cm]
\herm{\mX^{(\f)}} \vy^{(\f)} 
&\hspace{-0.15cm} \herm{\mX^{(\f)}} \mX^{(\f)}
\end{bmatrix}.
\end{equation}
Applying a well-known formula for the inverse of a block matrix \ \cite[Fact 2.17.3]{bernstein09} to $\specdensmatrix(\theta_\f)$ 
\vspace{-1mm}
yields 
$
%\label{equ_matrix_inversion_block_mtx_parameter_vector}
\big[\specdensmatrix^{-1}(\theta_\f)\big]_{2:p,1} 
%= - \inv{(\herm{\mX^{(\f)}}\mX^{(\f)})} \mX^{(\f)} \vy^{(\f)} (\herm{\vy^{(\f)}} \vy^{(\f)}  - \herm{\vy^{(\f)}}\mX^{(\f)} \inv{(\herm{\mX^{(\f)}}\mX^{(\f)})} \herm{\mX^{(\f)}} \vy^{(\f)} )^{-1} 
%= - \pinv{\mX^{(\f)}} \vy^{(\f)} \underbrace{ (\herm{\vy^{(\f)}} (\mI - \mX^{(\f)} \pinv{\mX^{(\f)}}) \vy^{(\f)})^{-1}  }_{
%= \big[ \big( \mathbf{C}^{(\f)}\big)^{-1} \big]_{1,1} \defeq  \beta} 
%\stackrel{\eqref{equ_def_parameter_vector_given_cov_matrix}}{=} 
=
- \omega \vbe^{(\f)} 
$%\end{equation}
with 
$
\omega \defeq \big(\herm{\vy^{(\f)}}  \!\vep^{(\f)}\big)^{-1}.\rmv
$
%\[
%\omega \defeq \big(\herm{\vy^{(\f)}} (\mI - \mX^{(\f)} \pinv{\mX^{(\f)}}) \vy^{(\f)}\big)^{-1}.
%\]
It also follows 
from \cite[Fact 2.17.3]{bernstein09} that $\omega= \big[ \specdensmatrix^{-1}(\theta_\f) \big]_{1,1} $ 
\pagebreak %%%%%%
and hence $\omega>0$ by  \eqref{equ_uniform_bound_eigvals_specdensmatrix}, which allows us to conclude that  $
\supp([\inv{\specdensmatrix}(\theta_\f)]_{2:p,1}) = \supp(\vbe^{(\f)})
$. 
%and therefore $\omega>0$ as we assume $\specdensmatrix(\theta_\f)$ to be positive definite (cf.~\eqref{equ_uniform_bound_eigvals_specdensmatrix}), which implies that also $\inv{\specdensmatrix}(\theta_\f)$ is positive definite. 
%This concludes the proof of \eqref{eq:suppeqbeta}. 
\end{proof}

\vspace{-0.1cm}
%It now follows from \eqref{eq:suppeqbeta} that the problem of determining the neighborhood $\mc N(1)$ can be formulated as that of finding the joint support 
%\CommentKeep{\emph{Is the joints support based on the observations?} Da die observations von dem CIG abhängen, hängt der joint support von den observations ab.}of the $\vbe^{(\f)}, f \in [F]$, (cf.~\eqref{eq:eqn1}).  %i.e., the union of the supports of the $\vbe^{(\f)}, f \in [F]$) which satisfy\eqref{eq:gmmvform}. 

The essence of Proposition \ref{prop:ladf} is that it reduces the problem of determining the neighborhood $\mc N(1)$ to that of finding the joint support 
%\CommentKeep{\emph{Is the joints support based on the observations?} Da die observations von dem CIG abhängen, hängt der joint support von den observations ab.}
of the $\vbe^{(\f)}, f \in [F]$.  
%The relation \eqref{eq:eqn1} follows from the proposition below. 
%
Recovering the $\vbe^{(\f)}$ based on the observations \eqref{eq:gmmvform} is now recognized as a \emph{multitask learning} or \emph{generalized multiple measurement vector problem} \cite{BuhlGeerBook,Lounici09,HeckelGMMVAllerton}, which in turn is a special case (with additional structure) of a block-sparse signal recovery problem \cite{BlockSparsityEldarTSP,MishaliEldar2008,EldarRauhut2010}. %\cite{HeckelGMMVAllerton}. 
Specifically, a multi-task learning problem can be cast as a block-sparse signal recovery problem by stacking the individual linear models in \eqref{eq:gmmvform} into a single linear model; the resulting system matrix $\diag\{ \mX^{(\task)}  \}_{\task \in [\nrtasks]}$ is block-diagonal. The approach described in \cite{Nowak2011} for VAR processes, albeit leading to a block-sparse recovery problem, does not result in a block-diagonal system matrix.

%from the $\vy^{(\f)}$, given the $\mX^{(\f)}$ is a multitask learning problem \cite{Lounici09,Argyriou07convexmulti-task,BuhlGeerBook,ObozWainJor11}, or a (particular form of the) block-sparse recovery problem \cite{BlockSparsityEldarTSP}. In fact, this recovery problem is a generalization of the multiple measurement vector (MMV) \cite{HeckelGMMVAllerton}. 
%An efficient method for recovering the parameter vectors $\vbe^{(\f)}$ from the measurements \eqref{eq:gmmvform} is the mLASSO, which can be formulated as follows (e.g., \cite{BuhlGeerBook}): 
An efficient method for solving the multi-task learning problem at hand is the mLASSO, which can be formulated as follows 
\vspace{-1mm}
(e.g., \cite{BuhlGeerBook}): 
%Specializing the group LASSO to the multitask model \eqref{eq:gmmvform} yields %the \emph{multitask LASSO} \cite{BuhlGeerBook,Lee_adaptivemulti-task}
\begin{align}
\hat \vbe = \argmin_{\vbe \in \mathbb C^{F(p-1)}}  
\bigg\{
\frac{1}{F} \sum_{f \in [F]} \norm[2]{\vy^{(\f)} - \mX^{(\f)} \vbe^{(\f)} }^2  + \lambda \norm[2,1]{\vbed}
\bigg\} \nonumber\\[-2mm]
\label{equ_multitask_lasso_graph_sel_statistic}\\[-8.5mm]
\nonumber
\end{align} 
where $\lambda>0$ is the LASSO parameter, 
$
\vbe \defeq \big( {{\vbe}^{(1)}}^{T} \!\!\cdots\,\break {{\vbe}^{(F)}}^{T} \big)^{T} \in \mathbb{C}^{F(p-1)}, 
$ 
 and $\norm[2,1]{\vbed} \defeq \sum_{r \in [p-1]}  \norm[2]{ \vbed_{r} }$ with $\vbe_r \in \complexset^F$ given by $[ \vbe_r ]_{f} \defeq [  \vbe^{(\f)} ]_{r}$. 
To compute the estimate $\hat \vbe$, one does not need to compute $\vy^{(\f)}$ and $\mX^{(\f)}$ by taking the square root of $\specdensmatrix(\theta_\f)$ as in \eqref{eq:defyX}. 
%The estimate $\hat \vbe$ obtained through \eqref{equ_multitask_lasso_graph_sel_statistic} is completely determined by the matrices $\specdensmatrix(\theta_\f)$. 
To see this, we note that  \eqref{equ_multitask_lasso_graph_sel_statistic} is equivalent 
\vspace{-2mm}
to % simple manipulations of the cost function yield 
\begin{align}
\label{equ_def_multitask_LASSO_covariance_matrix1}
\hat{\vbe}& = \! \argmin_{\vbe \in \mathbb C^{F(p-1)}} \bigg\{ \frac{1}{\nrtasks}  \sum_{\task \in [\nrtasks]}   \hspace*{-1mm}  \big[   {{\vbe}^{(\task)} }^{H} \herm{\mX^{(\f)}} \mX^{(\f)} {\vbe}^{(\task)} 
 \nonumber \\[-1mm]
 & \hspace*{16mm} \!-\! 2 \Re \big\{ \herm{\vy^{(\f)}} \mX^{(\f)} {\vbe}^{(\task)} \big\}  \big] +  \lambda \norm[2,1]{\vbed} \bigg\}  \\[-8.5mm]
\nonumber
\end{align} 
%where we used $\| \vbed' \|_{2,1}\!\defeq\!\sum_{r \in [q]}  \| \vbed'_{r} \|_{2}$ with $\vbe_r \defeq \transp{ \big( \entryop{\vbed^{(1)}}{r}, \ldots, \entryop{\vbed^{(\nrtasks)}}{r}  \big)} \! \in \! \mathbb{C}^{\nrtasks}$. 
%
and, by  \eqref{equ_constrution_multitask_model_cov_matrix}, $\herm{\vy^{(\f)}} \mX^{(\f)}$ and $\herm{\mX^{(\f)}} \mX^{(\f)}$ are submatrices of $\specdensmatrix(\theta_\f)$. Therefore, working with \eqref{equ_def_multitask_LASSO_covariance_matrix1} instead of  \eqref{equ_multitask_lasso_graph_sel_statistic} has the advantage that the square root $\specdensmatrix^{1/2}(\theta_\f)$ does not need to be computed in order to determine $\hat \vbe$. % based on $\specdensmatrix(\theta_\f)$. 

In summary, we have shown that the neighborhood $\mathcal N(1)$ can be found via the support of the mLASSO
estimate \eqref{equ_def_multitask_LASSO_covariance_matrix1}. Recognizing that this estimate depends on  $\specdensmatrix(\theta_\f)$, which is unknown, motivates the following inference algorithm (for general $r$), which simply uses $\ESDM(\theta_f)$ instead of $\specdensmatrix(\theta_\f)$ in \eqref{equ_def_multitask_LASSO_covariance_matrix1}.

%We have shown that the neighborhood $\mathcal N(1)$ can be found via the support of the mLASSO estimate \eqref{equ_def_multitask_LASSO_covariance_matrix1}. This estimate, however, depends on $\specdensmatrix(\theta_\f)$ which is unknown. This motivates the following inference algorithm, which simply uses the estimate $\ESDM(\theta_f)$ instead of $\specdensmatrix(\theta_\f)$ in \eqref{equ_def_multitask_LASSO_covariance_matrix1}. 
%%
%%
\begin{algorithm}
Given the observation $\vx[1],...,\vx[N]$, the parameter $F$, the threshold parameter $\eta$, and the mLASSO parameter $\lambda$ (the choice of $\eta$ and $\lambda$ will be discussed in Section \ref{sec_var_sel_consist}), perform the following 
\vspace{1mm}
steps:

{\bf Step 1: }
For each $\f \in [F]$, compute the SDM estimate $\ESDM(\theta_f)$ according 
\vspace{1mm}
to \eqref{equ_est_BT_sdm_theta_t}. % and set $\mZ^{(\f)} = \ESDM^{1/2}(\theta_f)$. 

{\bf Step 2:}
Compute the mLASSO estimate for each $r \in [p]$ 
\vspace{-1mm}
as 
%\begin{equation}
%\label{equ_lasso_graph_model_sel_scheme}
%\hat \vbe = \argmin 
%\left\{ 
%\frac{1}{F} \sum_{f \in [F]} \norm[2]{\vz_r^{(\f)} - \mZ^{(\f)} \vbe^{(\f)} }^2  + \lambda \norm[2,1]{\vbed}
%\right\}
%\end{equation}
%where the minimization is over $\vbe \in \mathbb C^{F(p-1)}$ such that $\beta_r^{(\f)} = 0$, for all $f\in [F]$, 
%and $\vz_r^{(\f)}$ is the $r$th column of $\mZ^{(\f)}$. 
\begin{align}
\label{equ_def_multitask_LASSO_covariance_matrix}
\hat{\vbe}& = \! \argmin_{\vbe \in \mathbb C^{F(p-1)}} \bigg\{ \frac{1}{\nrtasks}  \!\! \sum_{\task \in [\nrtasks]}   \hspace*{-1mm}  \big[  { {\vbe}^{(\task)}}^{H} \mG_r{(\f)} {\vbe}^{(\task)} 
 \nonumber \\[-1mm]
 & \hspace*{16mm} \!-\! 2 \Re \big\{ \herm{\vc_r^{(\f)}} {\vbe}^{(\task)} \big\}  \big] +  \lambda \norm[2,1]{\vbed} \bigg\} %\\[-8mm]
\end{align} 
where $\mG_r^{(f)} \in \mathbb C^{(p-1)\times (p-1)}$ is the submatrix of $\ESDM(\theta_f) \in \mathbb C^{p\times p}$ obtained by deleting its $r$th column and $r$th row, and $\vc_r^{(f)} \in \mathbb C^{(p-1)}$ is obtained by deleting the $r$th entry in the $r$th column 
\vspace{1mm}
of $\ESDM(\theta_f)$. 

{\bf Step 3:}
Estimate the neighborhood of node $r$ as the index 
\vspace{-.7mm}
set 
\begin{equation}
\label{equ_def_est_neighborhood_LASSO_CS_sel}
\widehat{\mathcal{N}}(r) = \big\{ r'  \, \big| \,   \norm[2]{\hat \vbe_{r'}}  > \eta \big\}
\end{equation}
\label{algo_CS_graph_sel}
with $\hat \vbe_{r'} \in \complexset^F$ given by $[\hat \vbe_{r'} ]_{f} = [  \hat \vbe^{(\f)} ]_{r'}$. 
\vspace{.5mm}
\end{algorithm}
Our algorithm can be regarded as a generalization of the algorithm proposed in \cite{MeinBuhl2006} for i.i.d.~random processes to general stationary random processes. 
The new element here is that since we consider \emph{general} Gaussian vector processes, we have \emph{multiple} measurements available to determine the CIG. This is exploited through the use of mLASSO instead of plain LASSO as employed in \cite{MeinBuhl2006}.

\vspace*{-1mm}
\section{Performance Guarantees} 
\label{sec_var_sel_consist} 
\vspace*{-2mm}

We now present conditions for our CIG selection scheme to correctly identify, with high probability, the neighborhoods $\mathcal{N}(r)$, and in turn the edge set $E$, of the underlying CIG. 
Our analysis yields allowed growth rates for the problem dimensions, i.e., the number $p$ of scalar process components and the maximum node 
degree $s_{\text{max}}$, as functions of the sample size $N$. %, for the CIG $\cig$ to be estimated correctly with high probability. 
Moreover, we provide concrete choices for the threshold parameter $\eta$ in \eqref{equ_def_est_neighborhood_LASSO_CS_sel} and the mLASSO parameter $\lambda$ in \eqref{equ_def_multitask_LASSO_covariance_matrix}.

A necessary and sufficient condition for mLASSO to correctly identify the joint support of the underlying parameter vector is the \emph{incoherence condition} \cite[Eqs.\ (4)--(5)]{BachConsistency2008}. 
This condition is a worst-case (in our case, over frequency $\theta_f$) condition \cite{HeckelGMMVAllerton} in that it needs the system matrices $\{ \mathbf{X}^{(\task)} \}_{\task \in [\nrtasks]}$ in \eqref{eq:gmmvform} (again, we consider node $r=1$) to be ``well-conditioned''  for \emph{all} $\task \in [\nrtasks]$. In other words, the incoherence condition does not predict any performance improvement owing to the availability of $\nrtasks$ measurements \eqref{eq:gmmvform} instead of just one. %having only one measurement at our disposal. 
%However, in our selection scheme will not use the support of the mLASSO itself but of a thresholded version (cf.\ \eqref{equ_def_est_neighborhood_LASSO_CS_sel}). 
We will therefore base our performance analysis on the \emph{multitask compatibility constant} \cite{BuhlGeerBook}, which is defined, for a given index set $\mathcal{S} \subseteq [p-1]$ of size $s$, as 
\vspace*{-2mm}
\begin{equation} 
\label{equ_def_multitask_compatibility_condition_bound}
 \phi(\mathcal{S}) \defeq \min_{ \vbed \in \mathbb{A}(\S)}   \frac{1}{\| \vbed_{\S} \|_{2,1}} \bigg( s\sum_{\task \in [\nrtasks]} \norm[2] {\mX^{(\task)} \vbed^{(\task)} }^{2} \bigg)^{1/2} 
\vspace*{-1mm}
\end{equation} 
with $\mathbb{A}(\S) \triangleq \big\{  \vbed \in  \mathbb{C}^{(p-1)\nrtasks}  \big| \| \vbed_{\S} \|_{2,1} > 0 \mbox{ and } \| \vbed_{\comp{\S}} \|_{2,1} \leq  3 \| \vbed_{\S} \|_{2,1} \}$. Here, $\vbe_\S \defeq \big( {\vbe_\S^{(1)}}^{T} \!\!\cdots\, {\vbe_\S^{(F)}}^{T} \big)^{T}$ where $\vbe_\S^{(f)}$ is the restriction of the vector $\vbe^{(f)}$ to the entries in $\S$.  
%As will be seen below (cf.\ \eqref{equ_def_spectral_compatibility_constant}), 
Invoking the concept of the multitask compatibility constant will be seen below to yield an average (across frequency $\theta_f$) requirement on the SDM $\specdensmatrix(\theta_f)$ for Algorithm \ref{algo_CS_graph_sel} to correctly identify the CIG. 

We start by defining the class $\mathcal M \!=\! \mathcal{M}(s_{\text{max}}, \rho_{\text{min}},\mu^{(h_{1})},\break \phi_{\text{min}},A,B)$ 
of $p$-dimensional, zero-mean, stationary, Gaussian processes $\mathbf{x}[n]$ with CIG $\cig=([p],E)$ of maximum node degree $s_{\text{max}}$ (cf.\ \eqref{equ_def_maximum_node_degree}) and SDM $\specdensmatrix(\theta) \in \mathbb{C}^{p \times p}$ 
satisfying \eqref{equ_uniform_bound_eigvals_specdensmatrix} and \eqref{equ_charact_edge_set_covariance_matrix_k}. The remaining parameters characterizing this class are defined as 
%% \pagebreak %%%%%
follows:
%Our results will apply to a class $\mathcal{M}$ of processes characterized by four parameters:
\vspace*{-1mm}
\begin{itemize} 
\item \emph{Minimum partial coherence} $\rho_{\text{min}}>0$: This parameter quantifies the minimum 
\pagebreak %%%%%
partial correlation between the spectral components of the process. In particular, we require that, for every $r \in [p]$,  $r' \in \mathcal{N}(r)$, 
% the process SDM $\specdensmatrix(\theta)$ to satisfy
\vspace*{-.8mm}
\begin{equation} 
\label{cond_rho_min}
  \sum_{\task \in [\nrtasks]} 
\left|
\frac{
 \big[ \specdensmatrix^{-1}(\theta_{\task}) \big]_{r,r'} 
}{
\big[\specdensmatrix^{-1}(\theta_{\task}) \big]_{r,r} 
} 
\right|^{2}  
\geq \rho^{2}_{\text{min}}.  \nonumber
\vspace*{-3mm}
\end{equation} 
\item \emph{ACF moment} $\mu^{(h_{1})}$: We quantify the spectral  smoothness of the processes in $\mathcal{M}$ using the ACF moment \eqref{equ_def_generic_moments_ACF} with weight function $h_{1}[m] \defeq  \left| 1- w[\lagvar](1-|\lagvar|/N) \right|$, where $w$ is the window function in \vspace*{-2mm}
\eqref{equ_est_BT_sdm_theta_t}. 
\item \emph{Minimum multitask compatibility constant}\footnote{
The relation between \eqref{equ_def_multitask_compatibility_condition_bound} and  \eqref{equ_def_spectral_compatibility_constant} is brought out by noting that, for $r=1$, $\norm[2] {\mX^{(\task)} \vbed^{(\task)} }^{2} = \herm{\vbed^{(\task)}}\herm{\mX^{(f)}}\mX^{(f)}\vbed^{(\task)}$ and  $ \herm{\mX^{(f)}}\mX^{(f)} = \mG_r^{(f)} $. 
}
 $\phi_{\text{min}}>0$ (cf.\ \eqref{equ_def_multitask_compatibility_condition_bound}): For every process in $\mathcal{M}$, we require%This parameter provides a lower bound on the multitask compatibility constant (cf.\ \eqref{equ_def_multitask_compatibility_condition_bound}) of the system matrices associated with the process (cf.\ Theorem \ref{thm_neighborhood_node_r_permuted_generalized_support_multitask_learning}):
\vspace*{-1.5mm}
\begin{align} 
\frac{1}{\| \vbed_{\mathcal{N}(r)} \|_{2,1} } \bigg( |\mathcal{N}(r)| \!\! \sum_{\task \in [\nrtasks]} { \vbed^{(\task)} }^{H} \mG_r^{(f)} \vbed^{(\task)} \! \bigg)^{1/2}
\hspace{-0.3cm}\geq \phimin 
\nonumber\\[-3mm]
\label{equ_def_spectral_compatibility_constant}\\[-8.5mm]
\nonumber
\end{align}
to hold for all $\vbed \in \mathbb{A}(\mathcal{N}(r))$ and all $r \in [p]$. 
\end{itemize} 
Combining techniques from large deviation theory \cite{RauhutFoucartCS} to bound the error $\| \ESDM(\theta_{\task}) - \specdensmatrix(\theta_{\task}) \|_{\infty}$ with a deterministic performance analysis of the mLASSO \cite{BuhlGeerBook}, one can derive the following result.
\vspace*{-0mm}
\begin{theorem} 
\label{thm_main_result_consistency_CS_graph_model}
Consider a process $\mathbf{x}[n]$ belonging to the class $\mathcal M$.
Let $\widehat{\mathcal{N}}(r)$ be the estimate of $\mathcal{N}(r)$ given by \eqref{equ_def_est_neighborhood_LASSO_CS_sel}, based on sample size $N$ and with the choices $\lambda =    \phi_{\emph{min}}^{2} \rho_{\emph{min}}/(18 s_{\emph{max}} \nrtasks)$ (in \eqref{equ_def_multitask_LASSO_covariance_matrix}) and $\eta=\rho_{\emph{min}}/2$ (in \eqref{equ_def_est_neighborhood_LASSO_CS_sel}). 
Then, if for some $\delta > 0$, the sample size $N$ and the ACF moment $\mu^{(h_{1})}$ satisfy
\vspace*{-1mm}
\begin{align} 
N  >  2^{8}  \log\bigg(\frac{4 \nrtasks p^3}{ \delta}\bigg)  \frac{\| w \|_{1}^{2} B^{2} s^{3}_{\emph{max}}}{\kappa^2}  \mbox{\; and \;} \mu^{(h_{1})} \! \leq \! \frac{\kappa}{2 s^{3/2}_{\emph{max}}}
\nonumber\\[-.5mm]
\label{eq:condonN}\\[-9mm]
\nonumber\end{align} 
with $\kappa \defeq (\phi_{\emph{min}}^{2}/ 174) \frac{\rho_{\emph{min}}}{\sqrt{\nrtasks}} \sqrt{A/B}$, 
the probability of Algorithm 1 delivering the correct edge set $E$ is at least $1-\delta$, i.e., % any neighborhood incorrectly is upper bounded by $\delta$, i.e., 
%\vspace*{-2mm}
%\begin{equation} 
$\prob \big\{  \bigcap_{r \in [p]} \{ \widehat{\mathcal{N}}(r)  = \mathcal{N}(r) \} \big\} \geq 1-\delta$. 
%\vspace*{-3mm}
%\end{equation}
%\label{thm:mainthm}
%\vspace*{-2mm}
\end{theorem} 
Theorem \ref{thm_main_result_consistency_CS_graph_model} shows that success is guaranteed with high probability if the sample size $N$ scales logarithmically in $p$ and 
polynomially in $s_{\text{max}}$, and if the process is sufficiently smooth, i.e., $\mu^{(h_{1})}$ is sufficiently small. 
%For a success probability of at least $1-\delta$, Theorem \ref{thm_main_result_consistency_CS_graph_model} requires the observed sample size $N$ to scale logarithmically in $p$ and 
%polynomially in $s_{\text{max}}$. %Moreover, the occurrence of the term $\| w \|_{1}^{2}$ in \eqref{eq:condonN} suggests that 
%%the window function should be well concentrated around $\lagvar = 0$. This, in turn via the second inequality in \eqref{eq:condonN} and the definition of $h_{1}[m]$, requires that the process 
%%$\mathbf{x}[n]$ has to be sufficiently smooth. 
%The second inequality in \eqref{eq:condonN} requires the process to be sufficiently smooth, i.e., to have a small ACF moment $\mu^{(h_{1})}$.

Let us particularize Theorem \ref{thm_main_result_consistency_CS_graph_model} to the special case of a VAR(1) process $\mathbf{x}[n]$ as considered in \cite{Bento2010,songsiri2010,Songsiri09}, i.e., 
$\mathbf{x}[n] = \mathbf{A} \mathbf{x}[n-1] + \mathbf{w}[n]$ with i.i.d.\ noise $\mathbf{w}[n] \sim \mathcal{N}(\mathbf{0}, \sigma^{2} \mathbf{I})$. As in \cite{Bento2010}, we 
take $\mathbf{A}$ to be the adjacency matrix of a dependency graph $\mathcal{D}$, which is related to---but in general different from---the CIG $\cig$. We assume that $\mathcal{D}$ is a simple graph of maximum 
node degree $d_{\text{max}}$ and the nonzero entries of the adjacency matrix are all equal to a single positive number $a \leq 1/(2d_{\text{max}})$. %For a VAR(1) process, the inverse SDM is given by $\specdensmatrix^{-1}(\theta) = \sigma^{-2} \big[ \mathbf{I} - \mathbf{A} \exp(-j \theta) \big]  \big[\mathbf{I} - \mathbf{A} \exp(j \theta)\big]$ \cite{Dahlhaus2000}. 
The VAR(1) process we consider 
satisfies \eqref{equ_uniform_bound_eigvals_specdensmatrix} and  \eqref{equ_def_maximum_node_degree} by its definition and belongs to the class $\processclass$ with $s_{\text{max}} = d_{\text{max}}^{2}$, $\rho_{\text{min}} = a$, $\phi_{\text{min}} = \sigma^{2}/4$, $A=\sigma^{2}/4$, and $B=4\sigma^{2}$. 
Moreover, condition \eqref{equ_charact_edge_set_covariance_matrix_k} 
\pagebreak %%%%%%
is satisfied as soon as $\nrtasks \geq 3$ since the entries of $\specdensmatrix^{-1}(\theta)$ are rational 
functions in $\exp(j \theta)$ with numerator degree $2$. The threshold in \eqref{eq:condonN} becomes 
$N >  C_{1} \log\big(\frac{4 \nrtasks p^3}{ \delta}\big)  \| w \|_{1}^{2} \nrtasks d^{6}_{\text{max}} /a^2$, with a constant $C_{1}$ that is independent of $\delta$, $p$, $d_{\text{max}}$, and $a$. 
In contrast, the corresponding threshold for the method in \cite{Bento2010} is $N > C_{2}  \log\big(\frac{4 d_{\text{max}} p^{2}}{ \delta}\big) d_{\text{max}}^{3}/a^{2}$, with a constant $C_{2}$ that is independent of $\delta$, $p$, $d_{\text{max}}$, and $a$. %These thresholds suggest that, for the special case of a VAR process, the method \cite{Bento2010} requires a considerably smaller sample size than our method to correctly identify the CIG. 
The difference between the growth rates of these thresholds with respect to $d_{\text{max}}$ may be explained by the fact that the method in \cite{Bento2010} is tailored to VAR processes whereas our approach applies to general (spectrally smooth) stationary processes.

\vspace{-1mm}

\section{Numerical Results}
\label{sec_numerical_results}

\vspace{-1.5mm}

%%%%%%%%%%%%%%%%%%%%%%%%%%%%%%%%%%%

\begin{figure}
\begin{center}
\hspace*{-2mm}\includegraphics{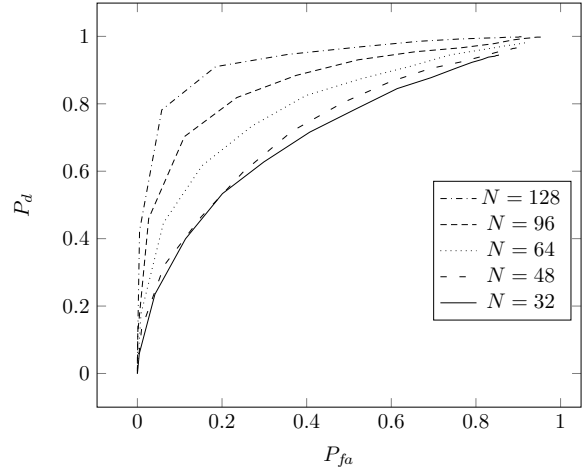}
\end{center}

%\begin{tikzpicture}[scale=0.8]
%\pgfplotsset{every axis legend/.append style={
%        at={(0.98,0.56)},}}
%    \begin{axis}[
%    	width=\textwidth*0.55,
%        	xlabel=$P_{f\!a}$,
%	ylabel=$P_{d}$,
%	]
%
%\addplot +[mark=none,dashdotted,black] table[x index=0,y index=1]{./AR128.dat}; 
%\addlegendentry{$N=128$};
%%
%\addplot +[mark=none,densely dashed,black] table[x index=0,y index=1]{./AR96.dat}; 
%\addlegendentry{$N=96$};
%%
%\addplot +[mark=none,dotted,black] table[x index=0,y index=1]{./AR64.dat}; 
%\addlegendentry{$N=64$};
%%
%\addplot +[mark=none,loosely dashed,black] table[x index=0,y index=1]{./AR48.dat}; 
%\addlegendentry{$N=48$};
%%
%\addplot +[mark=none,solid,black] table[x index=0,y index=1]{./AR32.dat}; 
%\addlegendentry{$N=32$};
%%
%
%
%%
%
%%
%%
%%\addplot +[mark=none,dashed] table[x index=0,y index=1]{./AR32_N_1-onlyAR.dat}; 
%%\addlegendentry{$N=32$, AR};
%%%
%%\addplot +[mark=none,dashed] table[x index=0,y index=1]{./AR64_N_1-onlyAR.dat}; 
%%\addlegendentry{$N=64$, AR};
%%%
%%\addplot +[mark=none,dashed] table[x index=0,y index=1]{./AR96_N_1-onlyAR.dat}; 
%%\addlegendentry{$N=96$, AR};
%%%
%%\addplot +[mark=none,dashed] table[x index=0,y index=1]{./AR128_N_1-onlyAR.dat}; 
%%\addlegendentry{$N=128$, AR};
%\end{axis}  
%\end{tikzpicture} 
\vspace*{-7mm}
  \caption{ROC curves for the compressive selection scheme.} %%  and $\sigma^{2}=1$
\label{fig_ROC}

\end{figure}

%%%%%%%%%%%%%%%%%%%%%%%%%%%%%%%%%%%

We generated\footnote{Matlab code to reproduce the results in this section is available at \rm{http://www.nt.tuwien.ac.at/about-us/staff/alexander-jung/}. } a Gaussian process $\mathbf{x}[n]$ of dimension $p=64$ by applying a finite impulse 
response (FIR) filter $g[\lagvar]$ of length $2$ to a zero-mean, stationary, white, Gaussian noise process $\mathbf{e}[n] \sim \mathcal{N}(\mathbf{0},\mathbf{C}_{0})$. %, i.e., 
The covariance matrix $\mathbf{C}_{0}$ was chosen such that the resulting CIG $\cig=([p],E)$ satisfies \eqref{equ_def_maximum_node_degree} with $s_{\text{max}}=3$.
The filter coefficients $g[\lagvar]$ are such that the magnitude of the associated transfer function is uniformly bounded from above and below by positive constants, thereby ensuring that conditions \eqref{equ_uniform_bound_eigvals_specdensmatrix} and \eqref{equ_charact_edge_set_covariance_matrix_k} (for arbitrary $F$) are satisfied.  
We then computed the estimates $\widehat{\mathcal{N}}(r)$ using Algorithm \ref{algo_CS_graph_sel} with window function $w[\lagvar] = \exp(- \lagvar^2 / 44 )$ and $\nrtasks = 4$.
We set $\lambda = c_{1} \phi^2_{\text{min}} \rho_{\text{min}} / (18 s_{\text{max}} \nrtasks)$ and $\eta = \rho_{\text{min}}/2$, where  $\phi_{\text{min}} = 0.0616$, $\rho_{\text{min}} = 0.5$, and $c_{1}$ was varied in the range $[10^{-3},10^{3}]$. 

In Fig.~\ref{fig_ROC}, we show receiver operating characteristic (ROC) curves with the average fraction of false alarms $P_{f\!a} \defeq \frac{1}{M} \sum_{i \in [M]} \frac{\sum_{(r,r') \notin E} I( r' \in \widehat{\mathcal{N}}_{i}(r))}{p(p-1)/2-|E|}$ and the average fraction of correct decisions $P_{d} \defeq \frac{1}{M} \sum_{i \in [M]} \frac{\sum_{(r,r') \in E} I( r' \in \widehat{\mathcal{N}}_{i}(r))}{|E|}$ for varying mLASSO parameter $\lambda$. Here, $\widehat{\mathcal{N}}_{i}(r)$ denotes the neighborhood estimate obtained from Algorithm \ref{algo_CS_graph_sel} in the $i$-th simulation run. We averaged over $M=10$ independent simulation runs. %Figure \ref{fig_ROC} shows that our CIG selection scheme yields reasonable performance for a $64$-dimensional process, even if $N=32$.

%\renewcommand{\baselinestretch}{0.9}\normalsize\footnotesize

%\bibliographystyle{IEEEbib}
%\bibliography{./LitAJ_ITC.bib,./tf-zentral}

\begin{thebibliography}{10}

\bibitem{Dahlhaus2000}
R.~Dahlhaus,
\newblock ``Graphical interaction models for multivariate time series,''
\newblock {\em Metrika}, vol. 51, pp. 151--172, 2000.

\bibitem{DahlhausEichler2003}
R.~Dahlhaus and M.~Eichler,
\newblock ``Causality and graphical models for time series,''
\newblock in {\em Highly {S}tructured {S}tochastic {S}ystems}, P.~Green,
  N.~Hjort, and S.~Richardson, Eds., pp. 115--137. Oxford Univ. Press, Oxford,
  UK, 2003.

\bibitem{BachJordan04}
F.~R. Bach and M.~I. Jordan,
\newblock ``Learning graphical models for stationary time series,''
\newblock {\em IEEE Trans. Signal Processing}, vol. 52, no. 8, pp. 2189--2199,
  Aug. 2004.

\bibitem{PHDEichler}
M.~Eichler,
\newblock {\em Graphical Models in Time Series Analysis},
\newblock Ph.D. thesis, Universit{\"a}t Heidelberg, Germany, 1999.

\bibitem{ElKaroui08}
N.~E. Karoui,
\newblock ``Operator norm consistent estimation of large dimensional sparse
  covariance matrices,''
\newblock {\em Ann. Statist.}, vol. 36, no. 6, pp. 2717--2756, 2008.

\bibitem{Santhanam2012}
N.~P. Santhanam and M.~J. Wainwright,
\newblock ``Information-theoretic limits of selecting binary graphical models
  in high dimensions,''
\newblock {\em IEEE Trans. Inf. Theory}, vol. 58, no. 7, pp. 4117--4134, Jul.
  2012.

\bibitem{RavWainLaff2010}
P.~Ravikumar, M.~J. Wainwright, and J.~Lafferty,
\newblock ``High-dimensional {I}sing model selection using
  $\ell_{1}$-regularized logistic regression,''
\newblock {\em Ann. Stat.}, vol. 38, no. 3, pp. 1287--1319, 2010.

\bibitem{Nowak2011}
A.~Bolstad, B.~D. Van~Veen, and R.~Nowak,
\newblock ``Causal network inference via group sparse regularization,''
\newblock {\em IEEE Trans. Signal Processing}, vol. 59, no. 6, pp. 2628--2641,
  Jun. 2011.

\bibitem{Bento2010}
J.~Bento, M.~Ibrahimi, and A.~Montanari,
\newblock ``Learning networks of stochastic differential equations,''
\newblock in {\em Proc.~Advances in Neural Information Processing Systems}, pp. 172--180, 2010.

\bibitem{MeinBuhl2006}
N.~Meinshausen and P.~B{\"u}hlmann,
\newblock ``High dimensional graphs and variable selection with the {L}asso,''
\newblock {\em Ann. Stat.}, vol. 34, no. 3, pp. 1436--1462, 2006.

\bibitem{FriedHastieTibsh2008}
J.~H. Friedmann, T.~Hastie, and R.~Tibshirani,
\newblock ``Sparse inverse covariance estimation with the graphical Lasso,''
\newblock {\em Biostatistics}, vol. 9, no. 3, pp. 432--441, Jul. 2008.

\bibitem{Songsiri09}
J.~Songsiri, J.~Dahl, and L.~Vandenberghe,
\newblock ``Graphical models of autoregressive processes,''
\newblock in {\em Convex Optimization in Signal Processing and Communications},
  Y.~C. Eldar and D.~Palomar, Eds., pp. 89--116. Cambridge Univ. Press,
  Cambridge, UK, 2010.

\bibitem{songsiri2010}
J.~Songsiri and L.~Vandenberghe,
\newblock ``Topology selection in graphical models of autoregressive
  processes,''
\newblock {\em J. Mach. Learn. Res.}, vol. 11, pp. 2671--2705,
  2010.

\bibitem{BlockSparsityEldarTSP}
Y.~C. Eldar, P.~Kuppinger, and H.~B{\"o}lcskei,
\newblock ``{B}lock-sparse signals: Uncertainty relations and efficient
  recovery,''
\newblock {\em IEEE Trans. Signal Processing}, vol. 58, no. 6, pp. 3042--3054,
  June 2010.

\bibitem{MishaliEldar2008}
M.~Mishali and Y.~C. Eldar,
\newblock ``Reduce and boost: Recovering arbitrary sets of jointly sparse
  vectors,''
\newblock {\em IEEE Trans. Signal Processing}, vol. 56, no. 10, pp. 4692--4702,
  Oct. 2008.

\bibitem{EldarRauhut2010}
Y.~C. Eldar and H.~Rauhut,
\newblock ``Average case analysis of multichannel sparse recovery using convex
  relaxation,''
\newblock {\em IEEE Trans. Inf. Theory}, vol. 56, no. 1, pp. 505--519, Jan.
  2009.

\bibitem{BuhlGeerBook}
P.~B{\"u}hlmann and S.~van~de Geer,
\newblock {\em Statistics for High-Dimensional Data},
\newblock Springer, New York, 2011.

\bibitem{Lounici09}
K.~Lounici, M.~Pontil, A.~B. Tsybakov, and S.~{van de Geer},
\newblock ``Taking advantage of sparsity in multi-task learning,''
\newblock in {\em Proc.~22nd Annual Conference on Learning Theory}, %Montreal, Canada, 
pp. 73--82, 2009.

\bibitem{Lee_adaptivemulti-task}
S.~Lee, J.~Zhu, and E.~P. Xing,
\newblock ``{A}daptive multi-task {L}asso: With application to e{QTL}
  detection,''
\newblock in {\em Proc.~Advances in Neural Information Processing Systems}, pp. 1306--1314, 2010.

\bibitem{Brillinger96remarksconcerning}
R.~Brillinger,
\newblock ``Remarks concerning graphical models for time series and point
  processes,''
\newblock {\em Revista de Econometria}, vol. 16, pp. 1--23, 1996.

\bibitem{stoi97}
P.~Stoica and R.~ Moses,
\newblock {\em Introduction to Spectral Analysis},
\newblock Prentice Hall, Englewood Cliffs, NJ, 1997.

\bibitem{bernstein09}
D.~S. Bernstein,
\newblock {\em Matrix Mathematics: Theory, Facts, and Formulas},
\newblock Princeton Univ. Press, Princeton, NJ, 2nd edition, 2009.

\bibitem{HeckelGMMVAllerton}
R.~Heckel and H.~B{\"o}lcskei,
\newblock ``Joint sparsity with different measurement matrices,''
\newblock in {\em Proc.~50th Allerton Conf.~Commun., Control, 
 and  Comput.}, %Monticello, IL, 
 pp. 698--702, 2012.

\bibitem{BachConsistency2008}
F.~R. Bach,
\newblock ``Consistency of the group Lasso and multiple kernel learning,''
\newblock {\em J. Mach. Learn. Res.}, vol. 9, pp. 1179--1225,
  2008.

\bibitem{RauhutFoucartCS}
S.~Foucart and H.~Rauhut,
\newblock {\em A Mathematical Introduction to Compressive Sensing},
\newblock Springer, New York, 2012.

\end{thebibliography}

\end{document}